\RequirePackage{fix-cm}
\documentclass[smallextended]{svjour3}       % onecolumn (second format)

\usepackage{smile}

\smartqed  % flush right qed marks, e.g. at end of proof
\usepackage{graphicx}
\usepackage{hyperref}
\usepackage[usenames]{color}
\newcommand{\bitem}{\begin{itemize}}
\newcommand{\eitem}{\end{itemize}}
\newcommand{\benum}{\begin{enumerate}}
\newcommand{\eenum}{\end{enumerate}}
\newcommand{\beq}{\begin{equation}}
\newcommand{\eeq}{\end{equation}}

\newcommand{\Ze}{\mathbf{0}}

\begin{document}

\title{\Large Optimization for Compressed Sensing: the Simplex Method and
	Kronecker Sparsification}

\titlerunning{Optimization for Compressed Sensing} % if too long for running head

%\author{
%Robert Vanderbei\thanks{Department of Operations Research and Financial Engineering, 
%	Princeton University, Princeton, NJ 08544, USA; e-mail: {\tt rvdb@princeton.edu} 
%	Research support by ONR Award N00014-13-1-0093} 
%\\ Han Liu\thanks{Department of Operations Research and Financial Engineering, 
%	Princeton University, Princeton, NJ 08544, USA; e-mail: 
%	{\tt hanliu@princeton.edu} Research supported by  NSF Grant III--1116730.} 
%\\ Lie Wang\thanks{Department of Mathematics, Massachusetts Institute of Technology, 
%	Cambridge, MA 02139, USA; e-mail: 
%	{\tt liewang@math.mit.edu} Research supported by  NSF Grant DMS-1005539.}
%\\ Kevin Lin\thanks{Department of Operations Research and Financial Engineering, 
%	Princeton University, Princeton, NJ 08544, USA; e-mail: {\tt klsix@princeton.edu} }
%}

\author{
Robert Vanderbei
\\ Han Liu
\\ Lie Wang
\\ Kevin Lin
	\thanks{The first author's research is supported by ONR  Award
		N00014-13-1-0093, the second author's by  NSF Grant
			III--1116730, and the third author's by NSF Grant
			DMS-1005539}
}

%\authorrunning{Short form of author list} % if too long for running head

\institute{Robert J. Vanderbei \at
	Department of Ops. Res. and Fin. Eng.,
	Princeton University,
	Princeton, NJ 08544. \\
              Tel.: +609-258-2345 \\
              \email{rvdb@princeton.edu}           %  \\
%             \emph{Present address:} of F. Author  %  if needed
}

\date{Received: date / Accepted: date}
% The correct dates will be entered by the editor

\maketitle

\begin{abstract}
In this paper we present two new approaches to efficiently solve
large-scale compressed sensing problems.
These two ideas are independent of each other and can therefore be used either
separately or together.  We consider all possibilities.

For the first approach, we note that the zero vector can be taken as the initial basic
(infeasible) solution for the linear programming problem and therefore, if the
true signal is very sparse, some variants of the simplex method can be
expected to take only a small number of pivots to arrive at a solution.  We
implemented one such variant and demonstrate a dramatic improvement in
computation time on very sparse signals.

The second approach requires a redesigned sensing mechanism  in which the
vector signal is stacked into a matrix.  This allows us to exploit the Kronecker compressed
sensing (KCS) mechanism.  We  show that the Kronecker sensing requires stronger
conditions for perfect recovery compared to the original vector problem.
However, the Kronecker sensing, modeled correctly, is a much sparser
linear optimization problem.  Hence, algorithms that benefit from sparse problem
representation, such as interior-point methods, can solve the Kronecker
sensing problems much faster than the corresponding vector problem.  In our numerical
studies, we demonstrate a ten-fold improvement in the computation time.

\end{abstract}

\section{Introduction.}
Compressed sensing aims to recover a sparse signal from a small number of
measurements.  The theoretical
foundation of compressed sensing was first laid out by \cite{Don06c}  and
\cite{CRT06} and can be traced further
back to the sparse recovery work of \cite{DS89, DH01,
DE03}.  More recent progress in the area of compressed sensing is summarized  in
\cite{EK12} and \cite{Ela10}.

Let  $\xb^{0} := (x^{0}_{1}, \ldots, x^{0}_{n})^{T}\in
\mathbb{R}^n$ denote a signal to be recovered.  We assume $n$ is large
and that $\xb^{0}$ is sparse.
Let $\Ab$ be a given (or chosen) $m \times n$ matrix with  $m < n$.
The  {\em compressed sensing problem} is to recover $\xb^{0}$ 
assuming only that we know $\yb = \Ab\xb^{0}$ and that $\xb^0$ is sparse.

Since  $\xb^{0}$  is  a sparse vector, one can hope that it is the sparsest
solution to the underdetermined linear system and therefore can be recovered
from  $\yb$ by solving
\[
({\rm P_0}) \qquad
\min_{\xb} \|\xb\|_0 \mbox{ subject to }  \Ab\xb = \yb,
\]
where $$\|\xb^{0}\|_0 := \#\{i : x_i \neq 0\}.$$
	
This problem is NP-hard due to the nonconvexity of the $0$-pseudo-norm.
To avoid the NP-hardness,   \cite{CDS98} proposed the {\em basis pursuit}
approach in which we use $\| \xb \|_{1}=\sum_j|x_j|$ to replace $\|\xb\|_{0}$:
\begin{eqnarray}
({\rm P_1}) \qquad
\min_{\xb} \|\xb\|_1 \mbox{ subject to }  \Ab\xb = \yb. \label{eq::P1}
\end{eqnarray}
\cite{DE03} and \cite{CDD09} have given conditions under which the solutions
to $({\rm P_{0})}$ and $({\rm P_{1}})$ are unique. 

One key question is: under what conditions are the solutions to 
$({\rm P_0})$ and $({\rm P_1})$ the same?
Various sufficient conditions have been 
discovered. 
%discovered: 
%see, e.g.,
%\cite{DH01, EB02, Tro04, CRT06, BGN08,  CDD09,
%DK10, Fou10}. 
For example, letting $\Ab_{*S}$ denote 
the submatrix of $\Ab$ with columns
indexed by a subset 
$S\subset \{1,\ldots, n\}$,
we say that $\Ab$  has the
$k$-{\em restricted isometry property ($k$-RIP)} with constant $\delta_{k}$ if
for any 
$S$ 
with cardinality $k$,
\begin{eqnarray}
(1-\delta_{k})\|\vb\|^{2}_{2} \leq \|\Ab_{*S}\vb\|^{2}_{2}
                              \leq (1+\delta_{k})\|\vb\|^{2}_{2}
			           ~~\text{for any}~\vb\in \mathbb{R}^{k},
\end{eqnarray}
where $\|\vb\|_2=\sqrt{\sum_{j=1}^n v_j^2}$.

We denote $\delta_{k}(\Ab)$ to be the smallest value of $\delta_k$ for which the
matrix $\Ab$ has the $k$-RIP property.  Under  the assumption that  $k:=\|\xb^{0}\|_{0}
\ll n$ and that  $\Ab$ satisfies the $k$-RIP condition,
	\cite{Cai2012} prove that whenever $\delta_{k}(\Ab)<1/3$,
the solutions to $({\rm P_{0})}$ and $({\rm P_{1})}$ are the same. Similar results have been obtained by  \cite{DT05a,DT05b, DT09} using convex geometric functional analysis. 

Existing algorithms for solving the convex program $({\rm P_{1})}$  include
interior-point methods \citep{CRT06, Boyd07}, projected gradient methods \citep{FNW07},
	and Bregman iterations \citep{Yin:2008}.  Besides solving the convex
	program $({\rm P_{1}})$, several greedy algorithms have been proposed,
	including matching pursuit \citep{MZ93} and its many variants
	\citep{Tro04,DET06,NV09,NT08,DMM09}. To achieve more scalability,
	combinatorial algorithms such as  HHS pursuit \citep{GSV07} and a
	sub-linear Fourier transform \citep{Iwe10} have also been developed.
%These
%algorithms are efficient but require significantly more measurements
%for perfect recovery.

In this paper, we revisit the optimization aspects of the classical compressed sensing formulation $({\rm P_{1}})$ and one of its extensions named Kronecker compressed sensing \citep{Duarte12}.  We consider two ideas for accelerating iterative algorithms---one can reduce the total number
of  iterations and one can reduce the  computation required to do one iteration.  The first method is competitive when $\xb^0$ is very sparse
whereas the second method is competitive  when it is somewhat less sparse. We back up these results by numerical simulations.

%Our first idea is based on the fact that, for the simplex-type methods,
%the zero vector can be taken as the initial basic solution.  This basic solution
%is infeasible but, if the pivot rule is chosen wisely, the optimal solution can
%be found in a very small number of pivots starting from this trivial starting
%point.  The 
%
%In particular, we use
%the {\em parametric simplex method} (see, e.g., \cite{Van07}) applied to a
%parametrically relaxed problem.
Our first idea is motivated by the fact that
the desired solution is sparse and therefore should require
only a relatively small number of simplex
pivots to find, starting from an appropriately chosen starting point---the zero
vector.
If we use
the {\em parametric simplex method} (see, e.g., \cite{Van07}) then it is easy to
take the zero vector as the starting basic solution.
%If the equality constraints in the model are relaxed and a
%parametrically-weighted penalty term is added to the objective function to gauge
%the deviation from equality, then the parametric simplex method can be employed
%to solve the problem in a small fraction of the number of pivots one would
%expect a naive implementation to require.

The second method requires a new sensing scheme. More specifically, we stack the signal vector $\xb$ into a matrix $\Xb$
and then multiplying the matrix signal on both the left and the right sides to get a
compressed matrix signal.  Of course, with this method we are changing the
problem itself since it is generally not the case that the original $\Ab$ matrix can be
represented as a pair of multiplications performed on the matrix associated with
$\xb$.  But, for many compressed sensing problems, it is fair game to redesign
the multiplication matrix as needed for efficiency and accuracy.  Anyway, 
%using ideas analogous to one step of a fast Fourier transform, 
this idea allows one to formulate the linear programming problem in such a way
that the constraint matrix is very sparse and therefore the problem can be
solved very efficiently.
%For our third method, we break up
%the left- and right-hand multiplications into separate smaller linear programming
%problems.  This last method is very fast but requires much stricter assumptions
%of sparsity for the method to correctly recover the signal.
This  results in a {\em Kronecker compressed sensing} (KCS) problem which has been
considered before (see \cite{Duarte12}) although we believe that the sparse representation of
the linear programming matrix is new.

Theoretically, KCS involves a tradeoff between  computational
complexity and informational complexity: it gains computational advantages at the
price of requiring more measurements (i.e., larger $m$). 
More specifically, in later sections, we show that, using sub-Gaussian random sensing matrices, whenever
\begin{eqnarray}
m\geq 225 k^2(\log(n/k^{2}))^2,
\end{eqnarray}
we recover the true signal with probability at least $1-4\exp(-0.1\sqrt{m})$.
It is easy to see that this scaling of $(m, n, k)$  is tight by considering the special case when all the nonzero entries of $\xb$ form a continuous block.

The rest of the paper is organized as follows. 
In the next section, we describe how to solve the vector version of the sensing
problem $({\rm P}_1)$ using the parametric simplex method.
Then, in Section 3, we describe
the main idea behind Kronecker compressed sensing (KCS).  
%introduce two variants: coupled KCS and decoupled KCS. In section 3, we
%describe the parametric simplex method and explain how to apply it to solve the
%original compressed sensing problem and the Kronecker compressed sensing
%problems.  
Numerical comparisons and discussion are provided in Section 4. 
%Some proofs are presented in the appendix.

\section{Vector Compressed Sensing via the Parametric Simplex Method}

Consider the following parametric perturbation to $({\rm P_1})$:
\begin{eqnarray} \label{eq::parametric}
    \hat{\xb} := \argmin_{\xb} \; \| \xb \|_1
			    +  \lambda \| \bepsilon \|_1
    \\
	     \begin{array}{rcl}
	     ~~\text{subject to}~~
	      \Ab \xb + \bepsilon &  =  & \yb 
	     \end{array}
	     \nonumber
\end{eqnarray}
where we introduced a parameter, $\lambda$.  Clearly for $\lambda = 0$ this
problem has a trivial solution:  $\hat{\xb} = \zero$.  And, as $\lambda$
approaches infinity, the solution approaches the solution of our original
problem $({\rm P_1})$.  In fact, for all values of $\lambda$ greater than some
finite value, we get the solution to our problem.

We could solve the problem with the parameter $\lambda$ as shown, but we prefer
to start with large values of the parameter and decrease it to zero.  So, we let
$\mu = 1/\lambda$ and consider this parametric formulation:
\begin{eqnarray} \label{eq::parametric}
    \hat{\xb} := \argmin_{\xb} \; \mu \| \xb \|_1
			    +  \| \bepsilon \|_1
    \\
	     \begin{array}{rcl}
	     ~~\text{subject to}~~
	      \Ab \xb + \bepsilon &  =  & \yb .
	     \end{array}
	     \nonumber
\end{eqnarray}
%We have relaxed the fundamental equality and we have added a penalty
%term to the objective function that penalizes in proportion to the degree to
%which the original equations are not met.  
For large values of $\mu$, the optimal
solution has $\hat{\xb} = \zero$ and
$\hat{\bepsilon} = \yb$.  For values of
$\mu$ close to zero, the situation reverses: $\hat{\bepsilon} = \zero$.

Our aim is to reformulate this problem as a parametric linear programming
problem and
solve it using the parametric simplex method (see, e.g., \cite{Van07}). In particular, we set parameter
$\mu$ to start at $\mu = \infty$ and successively reduce the value of $\mu$
for which the current basic solution is optimal until arriving at a value of
$\mu$ for which the optimal solution has
$\hat{\bepsilon} = 0$ at which point we will have solved the original problem.
If the number of pivots are few, then the final vector $\hat{\xb}$ will be mostly
zero.

It turns out that the best way to reformulate the optimization problem in \eqref{eq::parametric} as a linear
programming problem is to split each variable into a difference between two nonnegative variables,
\[
\xb = \xb^+ - \xb^-
~~\text{and}~~
\bepsilon = \bepsilon^+ - \bepsilon^- ,
\]
where the entries of $\xb^{+}, \xb^{-}, \bepsilon^{+}, \bepsilon^{-}$  are all nonnegative. 

The next step is to replace $\| \xb \|_1$ with $\one^T (\xb^+ + \xb^-)$ and to make
a similar substitution for $\| \bepsilon \|_1$.   In general, the sum does not
equal the absolute value but it is easy to see that it does at optimality.
Here is the reformulated linear programming problem:
\begin{eqnarray*} \label{eq::parametric2}
	\min_{\xb^+, \xb^-, \bepsilon^+, \bepsilon^-}
	   \; \mu \one^T (\xb^+ + \xb^-)
	        + \one^T (\bepsilon^+ + \bepsilon^-) \\
	     \begin{array}{rcl}
	     ~~\text{subject to}~~ \hspace*{0.2in}
		\Ab (\xb^+ - \xb^-) + (\bepsilon^+ - \bepsilon^-)
		&  =  & \yb
	\\
	      \xb^+, \xb^-, \bepsilon^+, \bepsilon^- & \ge & 0.
	     \end{array}
\end{eqnarray*}
For $\mu$ large, the optimal solution has $\xb^+ = \xb^- = 0$, and
$\bepsilon^+ - \bepsilon^- = \yb$.  And, given that these latter variables are
required to be nonnegative, it follows that
\[
    y_i > 0 ~ \Longrightarrow ~ \epsilon^+_i > 0 ~ \text{and} ~ \epsilon^-_i = 0
\]
whereas
\[
    y_i < 0 ~ \Longrightarrow ~ \epsilon^-_i > 0 ~ \text{and} ~ \epsilon^+_i = 0
\]
(the equality case can be decided either way).
With these choices for variable values, the solution is feasible for all $\mu$
and is optimal for large $\mu$.  Furthermore, declaring the nonzero variables to
be {\em basic} variables and the zero variables to be {\em nonbasic}, we see
that this optimal solution is also a basic solution and can therefore serve as a
starting point for the parametric simplex method.

Throughout the rest of this paper, we refer to the problem described here as the
vector compressed sensing problem.

\section{Kronecker Compressed Sensing}

In this section, we introduce the Kronecker compressed sensing problem \citep{Duarte12}.
Unlike the classical compressed sensing problem which mainly focuses on 
vector signals, Kronecker compressed sensing can be used for sensing
multidimensional signals (e.g., matrices or tensors).  For example, given a sparse
matrix signal $\Xb^{0}\in \mathbb{R}^{n_{1}\times n_{2}}$, we can use two
sensing matrices $\Ab\in \mathbb{R}^{m_{1}\times n_{2}}$ and
$\Bb\in\mathbb{R}^{m_{2}\times n_{2}}$ and try to recover $\Xb^{0}$ from 
knowledge of $\Yb=\Ab\Xb^{0}\Bb^{T}$. It is clear that when the signal is
multidimensional, Kronecker compressed sensing is more natural than 
classical vector compressed sensing. Here, we would like to point out that,
sometimes even when facing vector signals, it is still beneficial to
use Kronecker compressed sensing due to its added computational
efficiency.

More specifically, even though the target signal is a vector
$\xb^{0}\in\mathbb{R}^{n}$, we may first stack it into a matrix
$\Xb^{0}\in\mathbb{R}^{n_{1}\times n_{2}}$  by putting each length $n_{1}$
sub-vector of $\xb^{0}$ into a column of $\Xb^{0}$. Here, without loss of generality,  we assume
$n=n_{1}\times n_{2}$.  We then multiply the matrix signal $\Xb^{0}$ on both the
left and the right by sensing matrices $\Ab$ and $\Bb$ to get a compressed
matrix signal $\Yb^{0}$. In the next section, we will show that 
%using ideas analogous to one step of a fast Fourier transform, 
we are able to solve this
Kronecker compressed sensing problem much more efficiently than the 
vector compressed sensing problem.  
%In this section, we introduce two variants
%of the Kronecker compressed sensing: a coupled version versus a decoupled
%version. The decoupled KCS is computationally easier than the coupled KCS, but
%require more stringent conditions for successful recovery. 

%We adopt the following notation.
%Let $\Xb_{*j}$ denote the $j^{\rm th}$
%column of the matrix $\Xb$  and $\Xb_{i*}$ be the $i^{\rm th}$ row of the matrix
%$\Xb$.
%Also, we define $\|\Xb\|_{0} = \sum_{j,k}I(\Xb_{jk}\neq 0)$ and $\|\Xb\|_{1}:=\sum_{j,k}|\Xb_{jk}|$.
When discussing matrices, 
we let $\|\Xb\|_{0} = \sum_{j,k} {\mathbf 1}(x_{jk}\neq 0)$ and
$\|\Xb\|_{1}:=\sum_{j,k}|x_{jk}|$.

Given a matrix $\Yb\in\mathbb{R}^{m_{1}\times m_{2}}$ and the sensing matrices 
$\Ab$ and $\Bb$, our goal is to recover the original sparse signal $\Xb^{0}$  
%we first consider a coupled Kronecker sensing method, which solves 
by solving the following optimization problem:
\begin{eqnarray}
({\rm P_2}) \qquad
\hat{\Xb} = \argmin\|\Xb\|_{1}~~\text{subject to}~~\Ab\Xb\Bb^{T} = \Yb. \label{eq::coupled}
\end{eqnarray}
Here, $\Ab$ and $\Bb$ are sensing matrices of size
$m_1\times n_1$ and $m_2\times n_2$, respectively.
Let $\xb = \mathrm{vec}(\Xb)$ and $\yb = \mathrm{vec}(\Yb)$, where the
$\mathrm{vec}()$ operator takes a matrix and concatenates its elements column-by-column
to build one large column-vector containing all the elements of the matrix.
In terms of $\xb$ and $\yb$, problem $({\rm P_2})$ can be rewritten as
\begin{eqnarray}
\mathrm{vec}(\hat\Xb)=\argmin\|\xb\|_{1}~~\text{subject to}~~\Ub\xb = \yb, \label{eq::coupled2}
\end{eqnarray}
where $\Ub$ is given by the $(m_1m_2)\times (n_1n_2)$ Kronecker product of $\Ab$ and $\Bb$:
\begin{displaymath}
    \Ub
    :=
    \Bb\otimes\Ab
    = 
    \left[ 
        \begin{array}{ccc}
	    \Ab b_{1 1}  & \cdots & \Ab b_{1n_2} \\
	          \vdots  & \ddots & \vdots        \\
	    \Ab b_{m_21} & \cdots & \Ab b_{m_2 n_2} 
	\end{array} 
    \right].
\end{displaymath}
In this way, (\ref{eq::coupled}) becomes a vector compressed sensing problem.

To analyze the properties of this  Kronecker sensing approach, we recall
the definition of the restricted isometry constant for a matrix. For any $m\times n$
matrix $\Ub$, the $k$-restricted isometry constant $\delta_k(\Ub)$ is
defined as the smallest nonnegative number such that for any $k$-sparse vector
$\bh\in R^n$,
\begin{eqnarray}
(1-\delta_k(\Ub))\|\bh\|_2^2\leq \|\Ub\bh\|_2^2\leq(1+\delta_{k}(\Ub))\|\bh\|_2^2.
\end{eqnarray}

Based on the results in \cite{Cai2012}, we have%Sharp RIP bound for sparse signal %and low-rank matrix recovery. Applied and Computational Harmonic Analysis, to appear.
\begin{lemma}[\cite{Cai2012}]\label{lemma::cai12}
Suppose $k=\|\Xb^{0}\|_{0}$ is the sparsity of matrix $\Xb^{0}$. Then if $\delta_k(\Ub)<1/3$, we have $\mathrm{vec}(\hat\Xb)=\xb^{0}$ or equivalently  $\hat\Xb=\Xb^{0}$.
\end{lemma}
For the value of $\delta_k(\Ub)$, by lemma 2 of \cite{Duarte12}, we know that
\begin{eqnarray}
1+\delta_k(\Ub)\leq (1+\delta_k(\Ab))(1+\delta_k(\Bb)). \label{eq::KroneckerRIP}
\end{eqnarray}
In addition, we define strictly a sub-Gasusian distribution as follows:
\begin{definition}[Strictly Sub-Gaussian Distribution]
We say a mean-zero random variable $X$ follows a {\em strictly sub-Gaussian
	distribution} with variance $1/m$ if it satisfies
\begin{itemize}
\item $\mathbb{E}X^{2} =\ds \frac{1}{m}$,
\item $\mathbb{E}\exp\left( tX \right) \leq \exp\left(\ds \frac{t^{2}}{2m} \right)$ for all ~$t\in\mathbb{R}$.
\end{itemize}
\end{definition}
It is obvious that the Gaussian distribution with mean $0$ and variance $1/m^{2}$ satisfies the above definition. The next theorem provides sufficient conditions that guarantees perfect recovery of the  KCS problem with a desired probability.
\begin{theorem} \label{thm::kcs}
Suppose matrices $\Ab$ and $\Bb$ are both generated by independent strictly sub-Gaussian entries with variance $1/m$.  Let $C>28.1$ be a constant.  Whenever
\begin{eqnarray}
m_{1}\geq C\cdot k \log\left(n_{1}/k \right)~~\text{and}~~m_{2}\geq C\cdot k \log\left(n_{2}/k \right),
\end{eqnarray}
the convex program $({\rm P_{2}})$ attains perfect recovery with probability
\begin{eqnarray}
\mathbb{P}\biggl(\hat{\Xb} = {\Xb}^{0} \biggr)\geq 1-\underbrace{2\exp\left( -\Bigl(0.239-\frac{6.7}{C}\Bigr) m_{1}\right)-2\exp\left( -\Bigl(0.239-\frac{6.7}{C}\Bigr) m_{2}\right)}_{\rho(m_{1}, m_{2})}.
\end{eqnarray}
\end{theorem}
\begin{proof}
From Equation \eqref{eq::KroneckerRIP} and  Lemma \ref{lemma::cai12}, it suffices to show that
\begin{eqnarray}
\mathbb{P}\left(\delta_{k}(\Ab) < \frac{2}{\sqrt{3}}-1~\text{and}~\delta_{k}(\Bb) < \frac{2}{\sqrt{3}}-1\right) \geq 1 - \rho(m_{1}, m_{2}).
\end{eqnarray}
This result directly follows from  Theorem 3.6 of \cite{Baraniuk10} with a careful calculation of constants.
\qed
\end{proof}
From the above theorem, we see that for $m_1=m_2=\sqrt{m}$ and $n_1=n_2=\sqrt{n}$, whenever the number of measurements satisfies
\begin{eqnarray} \label{eq::kcsrate}
m\geq 225 k^2(\log(n/k^{2}))^2,
\end{eqnarray}
we have $\hat{\Xb} = {\Xb}^{0}$ with probability at least $1-4\exp(-0.1\sqrt{m})$.

Here we compare the above result to that of vector compressed sensing, i.e.,
     instead of stacking the original signal $\xb^{0}\in\mathbb{R}^{n}$ into a
     matrix, we directly use a strictly sub-Gaussian sensing matrix
     $\Ab\in\mathbb{R}^{m\times n}$ to multiply on $\xb^{0}$ to get $\yb =
     \Ab\xb^{0}$. We then plug $\yb$ and $\Ab$ into the convex program $({\rm
		     P_{1}})$ in Equation \eqref{eq::P1}  to recover $\xb^{0}$. Following the same argument as in Theorem \ref{thm::kcs},  whenever
\begin{eqnarray}\label{eq::csrate}
m \geq 30k\log\left(n/k\right),
\end{eqnarray}
we have $\hat{\xb} = {\xb}^{0}$ with probability at least $1-2\exp(-0.1m)$.
Comparing \eqref{eq::csrate} with \eqref{eq::kcsrate}, we see that KCS requires more stringent conditions for perfect recovery.

\section{Sparsifying the Constraint Matrix}

The key to efficiently solving the linear programming problem associated with the Kronecker
sensing problem lies in noting that the dense matrix $\Ub$ can be
factored into a product of two very sparse matrices:
\[
    \Ub 
    = 
    \left[ 
        \begin{array}{ccc}
	    \Ab b_{1 1}  & \cdots & \Ab b_{1n_2} \\
	          \vdots  & \ddots & \vdots        \\
	    \Ab b_{m_21} & \cdots & \Ab b_{m_2 n_2} 
	\end{array} 
    \right]
    = 
    \left[ 
        \begin{array}{cccc}
	  \;\Ab \;  & \; \Ze \;   & \cdots & \; \Ze \\
	    \Ze     &    \Ab      & \cdots &    \Ze  \\
	    \vdots  &    \vdots   & \ddots &  \vdots \\
	    \Ze     &    \Ze      & \cdots &   \Ab
	\end{array} 
    \right]
    \left[ 
        \begin{array}{cccc}
	     b_{1 1} \Ib   &  b_{1 2}   \Ib & \cdots &  b_{1 n_2} \Ib \\
	     b_{2 1} \Ib   &  b_{2 2}   \Ib & \cdots &  b_{2 n_2} \Ib \\
	       \vdots       & \vdots          & \ddots & \vdots          \\
	     b_{m_2 1} \Ib &  b_{m_2 1} \Ib & \cdots &  b_{m_2 n_2} \Ib
	\end{array} 
    \right]
    =:
    \Vb \Wb
    ,
\]
where $\Ib$ denotes a $n_1 \times n_1$ identity matrix and $\Ze$ denotes a 
$m_1 \times m_1$ zero matrix.
The constraints on the problem are 
\[
    \Ub \xb + \bepsilon = \yb .
\]
The matrix $\Ub$ is usually completely dense.  But, it is a product of two very
sparse matrices: $\Vb$ and $\Wb$.  Hence, introducing some new variables, call
them $\zb$, we can rewrite the constraints like this:
\[
    \begin{array}{ccccccl}
        \zb     & - & \Wb \xb &   &           & = & 0 \\
        \Vb \zb &   &         & + & \bepsilon & = & \yb .
    \end{array}
\]
And, as before, we can split $\xb$ and $\bepsilon$ into a difference between their
positive and negative parts to convert the problem to a linear program:
\begin{eqnarray*} \label{eq::parametric2}
	\min_{\xb^+, \xb^-, \bepsilon^+, \bepsilon^-}
	   \; \mu \one^T (\xb^+ + \xb^-)
	        + \one^T (\bepsilon^+ + \bepsilon^-) \\
	     \begin{array}{rcccccl}
	     ~~\text{subject to}~~ \hspace*{0.2in}
		\zb     & - & \Wb (\xb^+ - \xb^-) &   &                            &  =  & 0 \\
		\Vb \zb &   &                     & + & (\bepsilon^+ - \bepsilon^-)&  =  & \yb \\
	      &&&& \xb^+, \xb^-, \bepsilon^+, \bepsilon^- & \ge & 0.
	     \end{array}
\end{eqnarray*}

This formulation has more variables and more constraints.  But, the constraint
matrix is very sparse.   For linear programming, sparsity of the constraint
matrix is a significant contributor to algorithm efficiency
(see \cite{Van90e}).

\section{Numerical Results}

For the vector sensor, we generated random problems using $m=1,\!122 \; = 33
		\times 34$ and
$n=20,\!022 \; = 141 \times 142$.
We varied the number of nonzeros $k$ in signal $\xb^0$ from $2$ to $150$.
We solved the straightforward linear
programming formulations of these instances using an interior-point solver
called {\sc loqo} (\cite{Van94f}).
We also solved a large number of instances of the parametrically formulated
problem using the parametric simplex method as outlined above.

\begin{figure}[ht!]
\begin{center}
\includegraphics[width=6in]{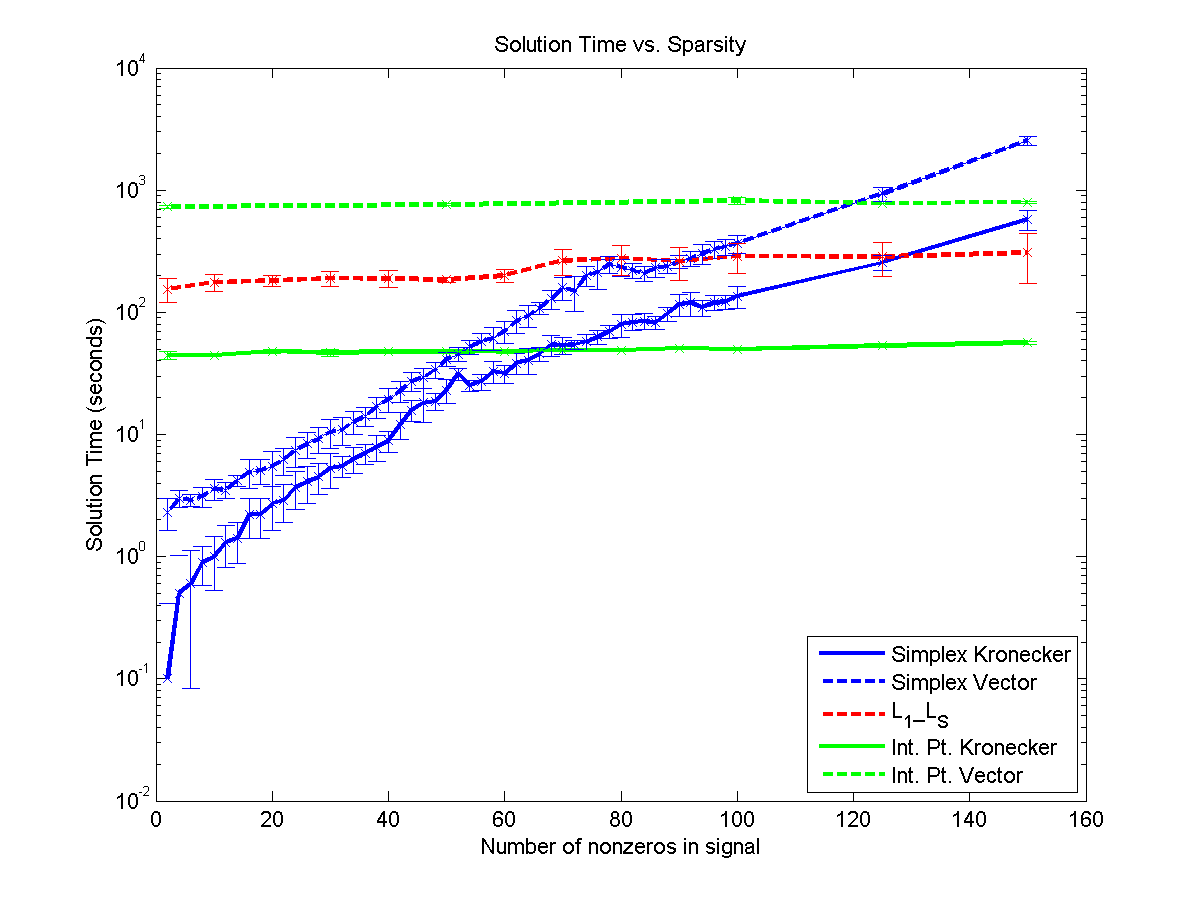}
\end{center}
\caption{Solution times for a large number of problem instances having
$m = 1,\!122$, $n = 20,\!022$, and various degrees of sparsity in
the underlying signal.  The horizontal axis shows the number of
nonzeros in the signal.  The vertical axis gives a semi-log
scale of solution times. The error bars have lengths equal to one standard deviation based on the multiple trials. }
\label{fig1}
\end{figure}

We followed a similar plan for the Kronecker (Matrix) sensor.  For these
problems, we used $m_1 = 33$, $m_2 = 34$, $n_1 = 141$, $n_2 = 142$, and
various values of $k$.  Again, the straightforward linear programming problems
were solved by {\sc loqo} and the parametrically formulated versions were solved
by a custom developed parametric simplex method.

For the Kronecker sensing problems, the matrices $\Ab$ and $\Bb$ were generated so
that their elements are independent standard Gaussian random variables.  For the
vector sensing problems, the corresponding matrix $\Ub$ was used.

We also ran the publicly-available, state-of-the-art $l_1\_l_s$ code 
(see \cite{Boyd07}).
%We also made an implementation of the decoupled matrix sensor and implemented a
%parametric simplex method for that too.

The results are shown in Figure \ref{fig1}.
The interior-point solver ({\sc loqo}) applied to the Kronecker sensing problem
is uniformly faster than both $l_1\_l_s$ and the interior-point solver applied
to the vector problem (the three horizontal lines in the plot).
For very sparse problems, the parametric simplex method is best.  In particular,
    for $k \le 70$, the parametric simplex method applied to the Kronecker
    sensing problem is the fastest method.  It can be two or three orders of
    magnitude faster than $l_1\_l_s$.  But, as explained earlier, the Kronecker
    sensing problem involves changing the underlying problem being solved.  If
    one is required to stick with the vector problem, then it too is the best
    method for $k \le 80$ after which the $l_1\_l_s$ method wins.

Instructions for downloading and running the various codes/algorithms described
herein can be found at
\url{http://www.orfe.princeton.edu/~rvdb/tex/CTS/kronecker_sim.html}.

%In terms of solving the linear programming problems
%$({\rm P_{1:LP}})$ and
%$({\rm P_{2:LP}})$ directly using an interior-point solver, the latter solves more
%than an order of magnitude faster.  Also, the solution time is virtually
%independent of the signal sparsity parameter $k$.
%When the signal is very sparse, say $k \le 70$, both of the parametric simplex
%codes beat the interior-point solver.  And, the improvement increases with
%sparsity.  With $k = 10$, the improvement is yet another order of magnitude.
%So, for $k=10$, the improvement over the original LP formulation of the vector
%sensor is almost a factor of $200$.  This improvement is dramatic and warrants
%further investigation.

%Finally, the decoupled problems solve extremely fast in both the regular LP
%method and in the parametric method.  However, beyond $k=8$, they don't
%correctly recover the original signal. The main reason is that in the first
%stage of $(\mathrm{P}_{3})$ we need to solve $m_{2}$ linear program subproblems
%to recover the columns of $\Wb$. Though each subproblem admits a perfect
%recovery with high probability, among all the $m_{2}$ subproblems, several of
%them can not be recovered well. These errors in the first stage are
%significantly manifested in the second stage.  This error cascading phenomenon
%leads to the inferior performance of the decoupled KCS compared to the coupled
%one.

\section{Conclusions}

We revisit compressed sensing from an optimization perspective. We advocate the
usage of  the parametric simplex algorithm for solving large-scale compressed
sensing problem.  The parametric simplex is a homotopy algorithm and enjoys many
good computational properties. We also propose two alternative ways for
compressed sensing which illustrate a tradeoff between computing and statistics.
In future work, we plan to extend the proposed method to the setting of 1-bit
compressed sensing.

------------------------------------------
------------------------------------------

\bibliographystyle{ims}
\bibliography{./local}
%\bibliography{../lib/refs_spacey}

% Non-BibTeX users please use

\end{document}